\documentclass{article}

\usepackage[utf8]{inputenc} 
\usepackage[T1]{fontenc}    %
\usepackage{hyperref}
\usepackage{amsmath}
\usepackage{amsthm}
\usepackage{amssymb}
\usepackage{xcolor}
\usepackage{algorithm}
\usepackage{algpseudocode}
\usepackage{cleveref}

\makeatletter
\theoremstyle{plain}
\newtheorem{thm}{\protect\theoremname}
\newtheorem{lem}[thm]{\protect\lemmaname}
\newtheorem{cor}[thm]{\protect\corollaryname}
\theoremstyle{definition}
\newtheorem{defn}[thm]{\protect\definitionname}
\newtheorem{example}[thm]{\protect\examplename}
\newtheorem{rem}[thm]{\protect\remarkname}

\makeatother
\providecommand{\corollaryname}{Corollary}
\providecommand{\definitionname}{Definition}
\providecommand{\lemmaname}{Lemma}
\providecommand{\theoremname}{Theorem}
\providecommand{\examplename}{Example}
\providecommand{\remarkname}{Remark}

\usepackage{authblk}
\author[1]{Naman Agarwal}
\author[2]{Alon Gonen}
\affil[1]{Google AI, Princeton}
\affil[2]{University of California San Diego}

\title{Optimal Sketching Bounds for Exp-concave Stochastic Minimization}

\begin{document}

\maketitle

\begin{abstract}
We derive optimal statistical and computational complexity bounds for exp-concave stochastic minimization in terms of the \emph{effective dimension}. For common eigendecay patterns of the population covariance matrix, this quantity is significantly smaller than the ambient dimension. Our results reveal interesting connections to sketching results in numerical linear algebra. In particular, our statistical analysis highlights a novel and natural relationship between algorithmic stability of empirical risk minimization and ridge leverage scores, which play significant role in sketching-based methods. Our main computational result is a fast implementation of a sketch-to-precondition approach in the context of exp-concave empirical risk minimization.
\end{abstract}

\section{Introduction}
The last two decades saw remarkable progress in the use of randomness
in the design of fast algorithms for numerical linear algebra (\cite{woodruff2014sketching,Kannan2017a}).
One of the principal ways to use randomness is to take a \emph{sketch-and-solve
}approach whereby an efficient (random) compression (a.k.a. sketch)
of a given matrix (or tensor) is computed such that performing the
required computation on the smaller object yields adequate accuracy.
Such techniques have significantly accelerated the runtime of performing
basic tasks such as (approximate) matrix multiplication and spectral
approximation both theoretically and practically.

The applications of the above techniques in data science arise through two key considerations. From a \textit{statistical}
point view, standard worst case characterizations of the \textit{sample
complexity (e.g. PAC complexity bounds)} often tend to be very pessimistic.
Integrating \textit{sketching} in this context seems promising due
to their ability to capture and exploit common distributional conditions
such as sparsity and low effective dimensionality.

The second consideration is \textit{computational} efficiency. Many important loss
functions in machine learning admit curvature which facilitates the
computation of high-accuracy solutions. However, if the data matrix
is ill-conditioned, simple first-order optimization methods fail to
utilize the underlying curvature. Evidently, a close variant of the
sketch-and-solve strategy, named \emph{sketch-to-precondition,} whereby  weak constant factor approximations of data covariance/Hessians are constructed to precondition the data/gradients, thus leading to alternative well-conditioned representations (\cite{clarkson2013low,Avron2017,gonen2016solving}). A key benefit of the sketch-to-precondition approach is that it leads to only a logarithmic dependence in the error. 

While there have been a number of successes in integrating sketching
methods in machine learning applications, few important and basic questions
are yet to be addressed. In this paper we consider the following two
challenges in the context of agnostic PAC based exp-concave stochastic minimization:
\begin{enumerate}
\item \textbf{Relating sketch-and-solve bounds to sample complexity
bounds: } Does there exist a systematic way to incorporate bounds achieved through the sketch-and-solve procedure into sample complexity? The sketch and solve bounds often depend on the effective dimensionality of the data rather than its true dimensionality which allows us to leverage the decay if any in the spectrum of the data. Furthermore, does achieving such bounds require incorporating
sketching methodologies directly into the learning algorithm itself or can such bounds be proved for generic regularized empirical risk minimization?

\item \textbf{Optimal Regularization for Sketch-To-Precondition  Optimization:} The choice of the regularization parameter in regularized risk minimization is often chosen optimally to minimize generalization error. From an optimization point of view  implementing the sketch-to-precondition algorithms for risk minimization
inherently introduces a fundamental trade-off between the cost of forming
the preconditioner and the resulting curvature. As we shall see, in
the context of empirical risk minimization, this tradeoff is controlled
by the choice of a certain \emph{regularization} parameter which is internally used by the sketch-to-precondition algorithm. While it is tempting(and is common practice) to use the same regularization parameter in the sketch-to-precondition algorithm, as the one introduced in the risk minimization problem, it might not be the optimal choice. Indeed the optimal choice for this parameter is problem dependent and a key question is whether we can find this optimal parameter while preserving the overall efficiency? 
\end{enumerate}
In this paper we answer both questions affirmatively in the context of agnostic PAC based exp-concave stochastic minimization. As for the first question, our statistical results hold for \emph{any} algorithm that minimizes the regularized empirical risk. Before describing our results formally, we now introduce the requisite terminology and define the problem setting we work with.

\subsection{Problem setting}
We study the two above questions in the relatively wide context of exp-concave
stochastic minimization. Specifically, let $\mathcal{W}\subseteq\mathbb{R}^{d}$
be a compact set of \emph{predictors} with $\ell_{2}$-diameter at
most $B>0$, $\mathcal{X}\subseteq\mathbb{R}^{d}$ be an \emph{instance}
\emph{space} such that $\mathcal{Z}:=\{w^{\top}x:w\in\mathcal{W},x\in\mathcal{X}\}\subseteq[-1,1]$,
and $\mathcal{Y}$ be a \emph{label set}.\footnote{i). For our computational results we will assume that projection onto
$\mathcal{W}$ can be implemented efficiently. ii) Our statistical
results have extremely mild dependence on $B$. iii) The assumption
on the magnitude of the prediction is arbitrary.} We consider the task of minimizing the \emph{expected risk}
\[
F(w)=\mathbb{E}_{(x,y)\sim\mathcal{D}}[\phi_{y}(w^{\top}x)]
\]
where for every $(x,y)\in\mathcal{X}\times\mathcal{Y}$, the function
$w\in\mathcal{W}\mapsto\phi_{y}(w^{\top}x)$ is $\alpha$-\emph{strongly convex}, i.e.
for all $z \in \mathcal{Z}$ $\phi_y''(z) \geq \alpha$.
We also assume that for every $y\in\mathcal{Y}$, $\phi_{y}$
is $\rho$-Lipschitz. The most prominent examples are the square
loss and the logistic loss. Actually, as pointed out in \cite{gonen2017average}, this framework includes many commonly used exp-concave and Lipschitz loss
functions.\footnote{One exception is the Log Loss which while being exp-concave is neither smooth or Lipschitz.} The following property of the loss functions is immediate.
\begin{lem}
For every $y\in\mathcal{Y}$, $\phi_{y}$ is $\gamma:=\alpha/\rho^{2}$
exp-concave.
\end{lem}
For our computational results we will also assume that for every $y$, $\phi_y$ is $\beta$-smooth.
\begin{example}
\label{exa:regression} Bounded $\ell_{2}$-regression: let $\mathcal{Y}=[-1,1]$
and let $\mathcal{W}$ and $\mathcal{X}$ be two compact sets in $\mathbb{R}^{d}$
such that $\forall w\in\mathcal{W}$ and $x\in\mathcal{X}$, $|w^{\top}x|\le1$.
The loss is defined by $\phi_{y}(z)=\frac{1}{2}(z-y)^{2}$. It is
easily verified that $\alpha=\beta=1$ and $\rho=2$. 
\end{example}
A learning algorithm $\mathcal{A}$ receives as an input an i.i.d. training
sample $S=((x_{i},y_{i})_{i=1}^{n})$ drawn according to $\mathcal{D}$, and outputs a predictor $\hat{w} \in \mathcal{W}$. Given a desired accuracy $\epsilon$, we denote by $n(\epsilon)$ the minimal size of a sample $S$ for which $\mathbb{E_{S\sim \mathcal{D}^{n}}}[F(\hat{w})]\le \epsilon$.\footnote{To obtain high-probability bounds (rather than in expectation) we
can employ the validation process suggested in \cite{Mehta2016}.} The function $n(\epsilon)$ is called the \textit{sample complexity} of $\mathcal{A}$. A popular practice is regularized loss minimization (RLM):
\begin{equation}
\hat{w}:=\arg\!\min_{w\in\mathcal{W}}\left\{ \hat{F}_{\lambda}(w):=\frac{1}{n}\sum_{i=1}^{n}\underbrace{\phi_{y_{i}}(w^{\top}x_{i})}_{=:f_{i}(w)}+\frac{\lambda}{2}\|w\|^{2}\right\}~,~~~\lambda := \frac{\epsilon}{B^2} \label{eq:regermpredictor}
\end{equation}
As shown by \cite{gonen2016average}, regularization is not needed for generalization purposes, i.e., empirical risk minimization (without regularization) is sufficiently stable in our setting. However, as we shall see, given a desired accuracy $\epsilon$, the regularization parameter helps us to effectively get rid of the dependence on the ambient dimension. More precisely, both our computational and statistical result depend on the \textit{effective-dimension} which is defined as follows. 
\begin{defn} \label{def:effDim}
Given a distribution $\mathcal{D}$ over $\mathcal{X}$, the effective dimension is defined as
\[
d_{\lambda}:=d_{\lambda} (\underbrace{\mathbb{E}[xx^\top]}_{=:C}) :=\sum_{i=1}^d \frac{\lambda_i(C)}{\lambda_i(C)+\lambda}~,
\]
where $\lambda_1\ge \ldots \ge \lambda_d \ge 0$ are the eigenvalues of $C=\mathbb{E}[xx^\top]$.
\end{defn}
\begin{rem}
For ease of presentation in the paper, we assume $d_{\lambda} \geq 1$ for all $\lambda$ considered in the paper. Our results can be equivalently modified by replacing $d_{\lambda}$ with $\max(d_{\lambda},1)$ everywhere.  
\end{rem}

We also denote the unregularized objective $\hat{F}_{0}(w)$ by $\hat{F}(w)$. The strong convexity of $\phi$ implies the following property of
the empirical loss (e.g. see Lemma 2.8 of \cite{Shalev-Shwartz2011a}). 
\begin{lem}
\label{lem:strongConvexity} Given a sample $S$, let $\hat{w}$
be as defined in \eqref{eq:regermpredictor} . Then for all $w\in\mathcal{W}$,
\[
\hat{F}_{\lambda}(w)-\hat{F}_{\lambda}(\hat{w})\ge\frac{\alpha}{2}(w-\hat{w})^{\top}\left(\frac{1}{n}\sum_{i=1}^{n}x_{i}x_{i}^{\top}+\frac{\lambda}{\alpha}I\right)(w-\hat{w})\,.
\]
\end{lem}

\paragraph{Kernels: } As we describe in Appendix \ref{sec:kernel}, all of our results hold in the Kernel setting.

\paragraph{Notation: }
As noted before, since many commonly used exp-concave functions satisfy our definitions with $\alpha,\beta,\rho = \Theta(1)$, our asymptotic notation ignores these parameters. We also use the notation $\tilde{O}$ to suppress logarithmic factors.

Besides the population covariance matrix  $C=\mathbb{E}[xx^\top]$ defined above, the following two matrices play a significant role in our development. Given an unlabeled sample $(x_1,\ldots,x_n)$, we define the scaled \emph{data matrix}, which contains the scaled vectors as rows, i.e. 
\[
A=[a_{1};\ldots;a_{n}]=n^{-1/2}[x_{1};\ldots;x_{n}]\in\mathbb{R}^{n\times d}
\]
Let $\hat{C}=\frac{1}{n}\sum_{i=1}^{n}x_{i}x_{i}^{\top}=A^\top A$ be the empirical covariance matrix. For any matrix $B$, we denote by $\mathrm{nnz}(B)$ the number of non-zero terms in $B$.

\subsection{Main Results}
\subsubsection{Statistical results}
Our first result is an upper bound on the sample complexity in terms of the effective dimension. The upper bound also exhibits a fast rate in terms of the desired accuracy $\epsilon$.
\begin{thm} \label{thm:upperSample}
The sample complexity of regularized loss minimization as defined in \eqref{eq:regermpredictor} satisfies
\[
n(\epsilon) = O \left(\frac{d_{\epsilon \alpha^{-1}B^{-2}}}{\epsilon} \right)
\]
\end{thm}
We also prove a nearly matching lower bound. We state and prove the lower bound in full generality in Appendix \ref{sec:lowerBound}. Here we only specify the bound for common decay patterns. We say that a given eigenvalue profile $\lambda_{1}\ldots\lambda_{d}\geq0$
satisfies polynomial decay if there exist $c,p>0$ such that
$\lambda_{i}\leq ci^{-p}$. Similarly, 
it satisfies exponential decay if there exists $c>0$
such that $\lambda_{i}\leq ce^{-i}$. 
\begin{thm} \label{thm:lowerBoundSample}
If the decay of the population matrix $C=\mathbb{E}[xx^\top]$ satisfies polynomial decay with degree $p$, then the sample complexity of \emph{any} learning algorithm for the case of ridge regression satisfies
\[
n(\epsilon) = \tilde{\Omega} \left(\frac{d_{\lambda'}}{\epsilon} \right)~,~~~\lambda' = (\epsilon/B^2)^{p/(p+1)}
\]
If the decay is exponential then
\[
n(\epsilon) = \tilde{\Omega} \left(\frac{d_{\epsilon B^{-2}}}{\epsilon} \right)
\]
\end{thm}
\paragraph{Implications to spectral approximation}
Interestingly, our lower bounds results imply identical lower bounds for spectral approximation (see exact definition in \cref{sec:preliminaries}). 
\begin{thm}
Let $A \in \mathbb{R}^{n \times d}$ and let $\lambda>0$ be any ridge parameter. Any $(\lambda,1/2)$ spectral approximation based on selecting (and scaling) a subset of rows of $A$ must include at least $d_{\lambda^{p/(p+1)}}(A^\top A)$ (respectively, $d_{\lambda}(A^\top A)$) rows if the decay is polynomial with degree $p$ (respectively, exponential).
\end{thm}
\begin{proof}
The proof follows by combining \cref{thm:lowerBoundSample} and \cref{lem:sketch-and-solve}.
\end{proof}

\subsubsection{Computational results}

As mentioned before, a common approach to solving the RLM objective \eqref{eq:regermpredictor} is to use the sketch-to-precondition approach. While it is well-known that this gives us a $\log(1/\epsilon)$ dependence in the required precision, the best known total runtime for this approach is
\begin{equation}
\label{eqn:exisitngstp}
  \tilde{O}(\mathrm{nnz}(A) + d_{\lambda}^2d) \text{ \qquad(see Section~\ref{sec:sketchToPrecond}).}  
\end{equation}
Note that the optimization runtime stated above depends on $\lambda$. It follows from the definition of the effective dimension that the larger is the value of $\lambda$, the faster is the convergence. However, $\lambda$ is usually chosen according to generalization considerations.

As we describe in Section \ref{sec:overalloptmain}, it turns out that we can hope to enjoy the best of the two worlds. That is, we can perform the optimization process using a larger regularization parameter $\lambda'$. To obtain guarantees with respect to the desired regularization parameter $\lambda$ we need to essentially repeat the said process for $\lambda'/\lambda$ iterations. This leads to a critical problem of efficient parameter tuning. Naively, to compute the best candidate $\lambda'$ we need to compute the associated effective dimension, $d_\lambda'$ which can be a costly computation by itself. Our main result remedies this issue by exhibiting an algorithm that uses an under-sampling approach for computing the best $\lambda'$ while avoiding computational overhead. The main guarantee our approach provides is summarized below. 

\begin{thm} \label{thm:mainCompRuntime}
Given an RLM objective \eqref{eq:regermpredictor} with regularization parameter $\lambda$, there exists an algorithm that finds an $\epsilon$-approximate solution to the RLM objective in time
\begin{equation}
\label{eqn:mainresult2}
\tilde{O}\left(\min_{\lambda' \geq \lambda}\left(\frac{\lambda'}{\lambda} \cdot \left(\mathrm{nnz}(A) + d_{\lambda'}^{2}d\right)\right)\right)
\end{equation}

\end{thm}
Please note that upto logarithmic factors, the above bound is always as good as \eqref{eqn:exisitngstp}. Furthermore, as the following example reveals, it can be significantly better.
\paragraph{Example: }
Fix a regularization parameter $\lambda>0$. Consider the case where the matrix $A$ has a constant number of eigenvalues of order $\Theta(1)$, whereas the rest of the eigenvalues are of order $\lambda$. We can see that the effective dimension $d_\lambda$ is proportional to $d$. On the other hand, letting $\lambda'= \sqrt{d} \lambda$, we have that $d_\lambda'=O(\sqrt{d})$. If $\mathrm{nnz}(A)=o(d^2)$, then the ratio between the runtimes is given by
\[
\frac{d_\lambda^2}{d_{\lambda'}^2} \cdot \frac{\lambda}{\lambda'} \approx d \cdot \frac{1}{\sqrt{d}} = \sqrt{d}~.
\]
While the example above is a toy example, a similar behaviour can be expected of any spectrum which shows a fast decay initially but a heavy tail towards the end.

\section{Related work}
\subsection{Sample complexity bounds}
To the best of our knowledge, the first bounds for empirical risk minimization for kernel ridge regression in terms of the effective dimension have been proved by \cite{zhang2003effective}. By analyzing the Local Rademacher complexity (\cite{bartlett2005local}), they proved an upper bound of $O(d_\lambda B^2/\epsilon)$ on the sample complexity. On the contrary, our bound have no explicit dependence on $B$ \footnote{The dependence on $B$ is implicit via $d_{\epsilon B^{-2}}$}. For instance, if the decay is exponential, our bounds exhibit only logarithmic dependence on $B$. More recently, \cite{Goel2017} used compression schemes (\cite{littlestone1986relating}) together with results on leverage score sampling from \cite{Musco2017} in order to derive a bound in terms of the effective dimension with no explicit dependence on $B$ . However, their rate is slow in terms of $\epsilon$.

Beside improving the above aspects in terms of accuracy rate and explicit dependence on $B$, our analysis is arguably simple and underscores nice connections between algorithmic stability and ridge leverage scores.

\paragraph{Online Newton Sketch}
The \emph{Online Newton Step} (ONS) algorithm due to \cite{hazan2007logarithmic} is a well-established method for minimizing exp-concave loss functions both in the stochastic and the online settings. As hinted by its name, each step of the algorithm involves a conditioning step that resembles a Newton step. Recent papers reanalyzed ONS and proved upper bounds on the regret (and consequently on the sample complexity) in terms of the effective dimension (\cite{luo2016efficient, calandriello2017second}). We note that using a standard online to batch reduction, the regret bound of \cite{luo2016efficient} implies the same (albeit a little weaker in terms of constants) sample complexity bounds as this paper. While ONS is certainly appealing in the context of regret minimization, in the statistical setting, our paper establishes the sample complexity bound \emph{irrespective} of the optimization algorithm used for the intermediate RLM step, thereby establishing that the computational overhead resulted by conditioning in ONS is not required.\footnote{We also do not advocate ONS for offline optimization, as it does not yield linear rate  (i.e., $log(1/\epsilon)$ iterations).}

\subsection{Sketching-based Bounds for Kernel Ridge Regression}
As we discussed above, the sketch-and-solve approach (e.g. see the nice survey by \cite{woodruff2014sketching}) has gained considerable attention recently in the context of enhancing both discrete and continuous optimization (\cite{luo2016efficient, gonen2015faster, gonen2016solving,clarkson2013low}). A recent paper by \cite{Musco2017} suggested to combine ridge leverage score sampling with the Nystr\"{o}m method to compute a spectral approximation of the Kernel matrix. As an application, they consider the problem of Kernel ridge regression and describe how this spectral approximation facilitates the task of finding $\epsilon$-approximate minimizer in time $O(ns^2)$, where $s=\tilde{O}(d_\lambda/\epsilon)$. Based on Corollary \ref{cor:kernelRegression} (with $n=O(d_\lambda/\epsilon)$), our complexity is better by factor of $\Omega (\min \{1/\epsilon^2,  d_\lambda/\epsilon\})$. A different application of the sketch-to-precondition approach, due to \cite{Avron2017}, focuses on polynomial Kernels and yields an algorithm whose runtime resembles our running time but also scales exponentially with the polynomial degree.

\section{Preliminaries} \label{sec:preliminaries}
\subsection{Sketching via leverage-score sampling \label{sec:leverage}}
In this section we define the notion of \emph{ridge leverage scores},
relate it to the effective dimension and explain how sampling according
to these scores facilitates the task of spectral approximation. Given
a sample $(x_{1},\ldots,x_{n})$, we define the scaled \emph{data
matrix} by 
\[
A=[a_{1},\ldots,a_{n}]=n^{-1/2}[x_{1};\ldots,x_{n}]\in\mathbb{R}^{n\times d}
\]
For a ridge parameter $\lambda>0$ we define the $i$-th \emph{ridge leverage score} by 
\[
\tau_{\lambda,i}=a_{i}^{\top}\left(A^{\top}A+\lambda I\right)^{-1}a_{i}.
\]
The following lemma gives a useful characterization of ridge
leverage scores. Informally, the $i$-th leverage score captures the
importance of the $x_{i}$ in composing the column space of the covariance
matrix $A^{\top}A$. The proof is detailed in Appendix \ref{sec:technical}. 
\begin{lem}
\label{lem:ridgeScore} For a ridge parameter $\lambda>0$ and for
any $i\in[n]$, $\tau_{\lambda,i}$ is the minimal scalar $t\ge0$
such that $a_{i}a_{i}^{\top}\preceq t(A^{\top}A+\lambda I)$. 
\end{lem}
A straightforward computation yields the following relation between the ridge leverage scores and the empirical effective dimension. 
\begin{lem}
For any ridge parameter $\lambda>0$, 
the effective dimension associated with the empirical covariance matrix $\hat{C}=A^\top A$ satisfies
\[
d_{\lambda}(\hat{C})=d_{\lambda}(A^{\top}A)=\sum_{i=1}^{n}\tau_{\lambda,i}
\]
\end{lem}
The notion of leverage scores gives rise to an algorithm for spectral
approximation that samples rows with probabilities proportional to
their ridge leverage scores. Before describing the sampling procedure,
we formally define the goal of spectral approximation. 
\begin{defn}
\textbf{(Spectral approximation)} We say that a matrix $\tilde{A}$
is a $(\lambda,\epsilon)$-spectral approximation to $A$ if 
\[
\frac{1-\epsilon}{1+\epsilon}(A^{\top}A+\lambda I)\preceq{\tilde{A}}^{\top}\tilde{A}+\lambda I\preceq A^{\top}A+\lambda I
\]
\end{defn}

\begin{defn}
\textbf{(Ridge Leverage Score Sampling)} \label{def:leverageSampling}
Let $(u_{i})_{i=1}^{n}$ be a sequence of ridge leverage score overestimates,
i.e., $u_{i}\ge\tau_{\lambda,i}$ for all $i$. For a fixed positive
constant $c>0$ and accuracy parameter $\epsilon$, define $p_{i}=\min\{1,c\epsilon^{-2}u_{i}\log d\}$
for each $i\in[n]$. Let $\texttt{Sample}(u,\epsilon)$ denote a function
which returns a diagonal matrix $S\in\mathbb{R}_{\ge0}^{n\times n}$,
where $S_{i,i}=((1+\epsilon)p_{i})^{-1/2}$ with probability $p_{i}$
and $0$ otherwise. 
\end{defn}

\begin{thm}[\cite{Mahoney2007,Musco2017}]
\label{thm:ridgeLeverageSampling}  Let
$(u_{i})_{i=1}^{n}$ be ridge leverage score overestimates, and denote by
$S=\texttt{Sample}(u,\epsilon)$.\footnote{We use the symbols $c\in(0,1),C\ge1$ to denote global constants.} 
\begin{enumerate}
\item With high probability, $SA$ is a $(\lambda,\epsilon)$-spectral approximation
to $A$. 
\item With high probability, $S$ has at most $\tilde{O}(\epsilon^{-2}\|u\|_{1})$
nonzero entries. In particular, if $\tau_{\lambda,i}\le u_{i}\le C\tau_{\lambda,i}$
for some constant $C>1$, then $S$ has at most $\tilde{O}(\epsilon^{-2}d_{\lambda})$
nonzero entries. 
\item There exists an algorithm which computes $(u_{i})_{i=1}^{n}$ with
$\frac{1}{2}\tau_{\lambda,i}\le u_{i}\le\tau_{\lambda,i}$ for all
$i$ in time $\tilde{O}(\mathrm{nnz}(A)+d_{\lambda}^{2}d)$ 
\end{enumerate}
\end{thm}

\subsection{Stability} \label{sec:stability}
In this part we define the notion of algorithmic stability, a common
tool to bound the generalization error of a given algorithm. We define
$\hat{w}_{i}$ to be the predictor produced by the regularized loss
minimization \eqref{eq:regermpredictor} when replacing the $i^{th}$ example with a fresh i.i.d.
pair $(x_{n+1},y_{n+1})$. We can now define the stability terms 
\[
\Delta_{i}:=f_{i}(\hat{w}_{i})-f_{i}(\hat{w})~\text{ and }~\Delta_{n+1}:=f_{n+1}(\hat{w})-f_{n+1}(\hat{w}_{i}).
\]
Note that both $\Delta_i$ and $\Delta_{n+1}$ are random variables which depend on the sample $S$ and the replacement $(x_{n+1},y_{n+1})$. The following theorem relates the expected generalization error to
the expected average stability. 
\begin{thm}
\textbf{(\cite{bousquet2002stability})} We have that \[\mathbb{E}_{S\sim\mathcal{D}^{n}}[F(\hat{w})-\hat{F}(\hat{w})]=\mathbb{E}\left[\frac{1}{n}\sum_{i=1}^{n}\Delta_{i}\right].\]
\label{thm:genToStab} 
\end{thm}

\section{Nearly Tight Sample Complexity Bound using Stability}

In this section we derive nearly tight sample complexity upper bound for exp-concave minimization based on the effective dimension. In particular we provide the proof of Theorem \ref{thm:upperSample} which follows immediately from the following more general thereom. 

\begin{thm}
For any $\lambda>0$ the excess risk of RLM is bounded as follows:
\[
\mathbb{E}_{\{(x_i,y_i)\}\sim\mathcal{D}^{n}}[F(\hat{w})-F(w^{\star})]\leq\frac{8\rho^{2}d_{\frac{\lambda}{\alpha}}(C)}{\alpha n}+\frac{\lambda}{2}B^{2}
\]
\label{thm:excessriskupperbound}
\end{thm}
\begin{proof} \textbf{(of Theorem \ref{thm:upperSample})}
Choosing $\lambda=\epsilon/(\alpha B^{2})$ yields the bound.
\end{proof}
We now prove Theorem \ref{thm:excessriskupperbound} in the rest of this section. At the heart of the proof lies a novel connection
between stability and ridge leverage scores. Following the notation presented in Section ~\ref{sec:stability}, fix a sample $S=((x_{1},y_{1}),\ldots,(x_{n},y_{n}))$ and consider replacing $(x_i,y_i)$ with an additional pair  $(x_{n+1},y_{n+1})$. Let $\lambda':=\lambda/\alpha$ and denote the corresponding
ridge leverage scores by $(\tau_{\lambda',i})_{i=1}^{n}$. We extend the definition for $\tau_{\lambda',n+1}$ as follows. Formally
\[ (\tau_{\lambda',i})_{i=1}^{n+1} := x_i^{\top}\left(\frac{1}{n}\sum_{j=0}^{n} x_jx_j^{\top} + \lambda'I\right)x_i.\]
Note that as defined above, the ridge leverage scores $\tau_{\lambda',i}$ are random variables depending on the sample. The following lemma underscores the connection between stability and leverage scores. 
\begin{lem}
\label{lem:stabilityViaRidgeScoes} With probability 1, the stability terms $\Delta_i$ and $\Delta_{n+1}$ satisfy
\[
\Delta_{i}\le\rho\cdot\sqrt{\frac{2\tau_{\lambda',i}(\Delta_{i}+\Delta_{n+1})}{\alpha}}~,~~\Delta_{n+1}\le\rho\cdot\sqrt{\frac{2\tau_{\lambda',n+1}(\Delta_{i}+\Delta_{n+1})}{\alpha}}
\]
\end{lem}

\begin{proof}
For a fixed $i$, consider the replacement of $(x_{i},y_{i})$ by a fresh pair $(x_{n+1},y_{n+1})$, and denote the modified
regularized risk by 
\[
\hat{F}_{\lambda,i}(w)=\hat{F}_{\lambda}(w)-\frac{1}{n}f_{i}(w)+\frac{1}{n}f_{n+1}(w).
\]
Further by the mean value theorem there exists $z=\delta\hat{w}+(1-\delta)\hat{w}_{i}$
with $\delta\in[0,1]$ such that $\Delta_{i}=\phi_{y_{i}}'(x_{i}^{\top}z)x_{i}^{\top}(\hat{w}_{i}-\hat{w}).$
Using Lipschitness we have that 
\begin{align*}
 & \Delta_{i}\le\rho\cdot|x_{i}^{\top}(\hat{w}_{i}-\hat{w})|=\rho\cdot\sqrt{(\hat{w}_i -\hat{w})^\top x_{i}x_{i}^\top(\hat{w}_{i}-\hat{w})}\\
 & \le\rho\cdot\sqrt{n\cdot\tau_{\lambda',i}\cdot(\hat{w}_{i}-\hat{w})^{\top}\left(\hat{C}+\lambda'I\right)(\hat{w}_{i}-\hat{w})} \qquad\qquad\qquad (\textrm{Lemma}~\ref{lem:ridgeScore})\\
 & \le\rho\cdot\sqrt{\frac{2n\cdot\tau_{\lambda',i}}{\alpha}}{\cdot}\sqrt{(\hat{F}_{\lambda}(\hat{w}_{i})-\hat{F}_{\lambda}(\hat{w}))} \qquad\qquad\qquad\qquad\qquad\; (\textrm{Lemma}~\ref{lem:strongConvexity})\\
 &=\rho\cdot\sqrt{\frac{2n\cdot\tau_{\lambda',i}}{\alpha}}{\cdot}\sqrt{\hat{F}_{\lambda,i}(\hat{w}_i) - \hat{F}_{\lambda,i}(\hat{w}) + \frac{f_{i}(\hat{w}_{i})-f_{i}(\hat{w})}{n} +  \frac{f_{n+1}(\hat{w})-f_{n+1}(\hat{w}_{i})}{n}}\\
 & \le\rho\cdot\sqrt{\frac{2n\cdot\tau_{\lambda',i}}{\alpha}}\cdot\sqrt{\frac{\Delta_{i}+\Delta_{n+1}}{n}}\\
 & =\rho\cdot\sqrt{\frac{2\cdot\tau_{\lambda',i}(\Delta_{i}+\Delta_{n+1})}{\alpha}}
\end{align*}
where the last inequality follows by the the optimality of $\hat{w}$
and $\hat{w}_{i}$ w.r.t. the original and the modified regurlized
risks, respectively. This concludes the first inequality. The second
inequality is proved analogously.
\end{proof}
Using this lemma the proof of Theorem~\ref{thm:excessriskupperbound}
uses Theorem \ref{thm:genToStab} in a standard way as follows.

\begin{proof}
\textbf{(of Theorem~\ref{thm:excessriskupperbound})} 
We first use Theorem \ref{thm:genToStab} to relate the excess risk
to the average stability: 
\begin{align*}
 & \mathbb{E}[F(\hat{w})-F(w^{*})]=\mathbb{E}[F(\hat{w})-\hat{F}(\hat{w})]+\mathbb{E}[\hat{F}(\hat{w})-F(w^{*})]\\
 & \le\mathbb{E}[F(\hat{w})-\hat{F}(\hat{w})]+\mathbb{E}[\hat{F}_{\lambda}(\hat{w})-\hat{F}_{\lambda}(w^{*})]+\frac{\lambda}{2}\left(\|w^*\|^{2} - \|\hat{w}\|^2\right)\\
 & \le\mathbb{E}[F(\hat{w})-\hat{F}(\hat{w})]+\frac{\lambda}{2}B^2 \qquad\qquad\qquad(\hat{w} \text{ is optimal})\\
 & \leq\mathbb{E}\left[n^{-1}\sum_{i=1}^{n}\Delta_{i}\right]+\frac{\lambda}{2}B^2 \qquad\qquad\qquad\;\;( \cref{thm:genToStab})
\end{align*}
Applying Lemma  \ref{lem:stabilityViaRidgeScoes} and using the inequality
$(a+b)^{2}\le2a^{2}+2b^{2}$ yields
\[
\Delta_{i}+\Delta_{n+1}\le\frac{4\rho^{2}(\tau_{\lambda',i}+\tau_{\lambda',n+1})}{\alpha}\Rightarrow\frac{1}{n}\sum_{i=1}^{n}(\Delta_{i}+\Delta_{n+1})\le\frac{4\rho^{2}}{\alpha n}\sum_{i=1}^{n}(\tau_{\lambda',i}+\tau_{\lambda',n+1})~.
\]
Since $\Delta_{i}$ and $\Delta_{n+1}$ (similarly, $\tau_{\lambda',i}$
and $\tau_{\lambda',n+1}$) are distributed identically, the
result now follows from the following lemma whose proof is provided in Appendix \ref{sec:stabilityProof}.
\begin{lem}
\label{lemma:expectatationdlambda} Let $x$ be a random variable
supported in a bounded subset of $\mathbb{R}^{d}$ with $\mathbb{E}[xx^{T}]=C$.
Let $\hat{C}=\frac{1}{n}\sum_{i=1}^{n}x_{i}x_{i}^{T}$ where $x_{i}$
are i.i.d copies of $x$. Then we have that for any fixed $\lambda>0$
\[
\mathbb{E}[d_{\lambda}(\hat{C})]\leq2d_{\lambda}(C)
\]
\end{lem}

\end{proof}

\section{Fast Sketched Preconditioning for Exp-concave Minimization}
\label{sec:overalloptmain}
In this section we prove our main result regarding the optimization of the RLM objective, i.e. Theorem~\ref{thm:mainComp}. 
We start by reviewing and extending an efficient preconditioning technique that provides
a reduction from exp-concave empirical risk minimization to constant
spectral approximation. For our computational results we will also
assume that $\phi_{y}$ is $\beta$-smooth.

\subsection{The Sketch-to-Precondition Method} \label{sec:sketchToPrecond}
 Let $((x_{1},y_{1}),\ldots,(x_{n},y_{n}))$
be a sample and let $A\in\mathbb{R}^{n\times d}$ be the scaled data
matrix defined above. For a ridge parameter $\lambda>0$, let $SA$
be a $(\lambda,1/2)$-spectral approximation of $A$ and consider
the change of variable: 
\begin{align}
 & \tilde{x}_{i}:=(A^{\top}S^{\top}SA+\lambda I)^{-1/2}x_{i}\,,\quad\tilde{f}_{i}(w):=\phi_{y_{i}}(w^{\top}\tilde{x}_{i})\label{eq:precondDef}\\
 & \tilde{F}_{\lambda}(w):=\frac{1}{n}\sum_{i=1}^{n}\tilde{f}_{i}(w)+\frac{\lambda}{2}\|w\|_{(A^{\top}S^{\top}SA + \lambda I)^{-1}}^{2} \nonumber \\ &\tilde{\mathcal{W}} = \{(A^{\top}S^{\top}SA+\lambda I)^{1/2}\tilde{w}\;\;|\;\; \tilde{w} \in \mathcal{W}\}\nonumber
\end{align}
where $\|x\|_M^2:= x^{\top}Mx$. Note that $\tilde{f}_{i}(w)=f_{i}((A^{\top}S^{\top}SA+\lambda I)^{-1/2}w)$.
Applying projected Gradient Descent (GD) to the transformed RLM and
mapping back yields Algorithm \ref{alg:sketchedCondGD}. The following
lemma, which can be seen as extension of \cite{clarkson2013low}[Theorem
43], is proved in Appendix~\ref{sec:technical}.
\begin{lem} \label{lem:sketchToPrecondRecap}
Algorithm \ref{alg:sketchedCondGD} finds an $\epsilon$-approximate
minimizer to the empirical risk $\hat{F}_{\lambda}(w)$ after $T=\tilde{O}(1)$
rounds. The overall runtime is bounded above by $\tilde{O}\left(\mathrm{nnz}(A)+d_{\lambda}^{2}d\right)$. 
\end{lem}

\begin{algorithm}
\caption{Preconditoned GD using Sketch-to-precondition}
\label{alg:sketchedCondGD} \begin{algorithmic} \State \textbf{Input:}
A sample $S=((x_{1},y_{1}),\ldots,(x_{m},y_{m}))$ \State \textbf{Parameter:}
$\lambda>0$ \Comment{Ridge parameter} \State Perform the change
of variable as in (\ref{eq:precondDef}) \State $\tilde{w}_{1}=0$
\Comment{Applying Projected Gradient Descent to $\tilde{F}_{\lambda}(w)$}
\For {$t=1$ to $T-1$} \State $\tilde{w}_{t+1}=\Pi_{\tilde{\mathcal{W}}}(\tilde{w}_{t}-\nabla \tilde{F}_\lambda(\tilde{w}_t))$ 
\EndFor 
\State $\hat{w}_{T}=(A^{\top}S^{\top}SA+\lambda I)^{-1/2}\tilde{w}_{t}$
\end{algorithmic} 
\end{algorithm}

\subsection{Tradeoff between Oracle Complexity and Effective Dimensionality}
 \label{sec:ridgeTune}
As mentioned in the introduction, given a ridge parameter $\lambda$ (which is chosen according to generalization concerns), it is possible to perform multiple steps of optimization with a different ridge parameter $\lambda'>\lambda$ in order to accelerate the optimization process. We first describe a method to do so. For this section we treat $\lambda$ as fixed and treat $\lambda'$ as a variable. 

\paragraph{The Proximal Point Algorithm: Overview}
the Proximal Point algorithm (PPA) due to \cite{frostig2015regularizing} allows us to reduce minimizing $\hat{F}_\lambda$ to minimization w.r.t. a different ridge parameter $\lambda' \geq \lambda$.
The basic idea is to repeat the minimization process for $\lambda'/\lambda$ epochs.
For a fixed $\bar{w}\in\mathcal{W}$ and $\lambda' \geq \lambda$, define \[\hat{F}_{\lambda',\bar{w}}(w):=\hat{F}_{\lambda}(w)+\frac{\lambda'-\lambda}{2}\|w-\bar{w}\|^{2}.\]
Suppose we start from $w_{0}=0$. At time $t$, for some fixed constant $c$, we find a point $w_{t}$
satisfying
\[
\hat{F}_{\lambda',w_{t-1}}(w_{t})-\min_{w\in\mathcal{W}}\hat{F}_{\lambda',w_{t-1}}(w)\le\frac{c\lambda}{\lambda'}\left(\hat{F}_{\lambda',w_{t-1}}(w_{t-1})-\min_{w\in\mathcal{W}}\hat{F}_{\lambda',w_{t-1}}(w)\right)\,.
\]
 \begin{lem} \label{lem:ppa}
\cite{frostig2015regularizing} Applying PPA with $\lambda'\ge\lambda$ yields $\epsilon$-approximate minimizer to $\hat{F}_\lambda$ after $t=\tilde{O}(\lambda'/\lambda)$
epochs, i.e., $\hat{F}_{\lambda}(w_{t})-\min_{w\in\mathcal{W}}\hat{F}_{\lambda}(w)\le\epsilon$.
\end{lem}
\begin{proof}
Apply Lemma 2.4 of \cite{frostig2015regularizing}with $\lambda=\lambda'$
and $\mu=\lambda.$
\end{proof}
Combining PPA (\cref{lem:ppa}) with Sketch-to-Precondition(Algorithm \ref{alg:sketchedCondGD}) and using  Lemma~\ref{lem:sketchToPrecondRecap} yields the following complexity bound for minimizing $\hat{F}_{\lambda}$ and any fixed $\lambda'$:
\begin{equation}
\label{eqn:ppatotaltime}
    \tilde{O}\left( \frac{\lambda'}{\lambda}  \cdot\left(\mathrm{nnz}(A) + d_{\lambda'}^{2}d\right)\right).
\end{equation}
Since we wish to perform (upto $\log$ factor) as well as the best choice of $\lambda'$,
therefore we wish to minimize the complexity term
\begin{equation} \label{eq:compLambda}
\psi(\lambda'):=\frac{\lambda'}{\lambda}\cdot(\mathrm{nnz}(A) + d_{\lambda'}^2d)
\end{equation}
over all possible $\lambda' \ge \lambda$. We define the following quantities, 
\[\psi^* = \min_{\lambda' \geq \lambda} \psi(\lambda') \qquad \lambda^* = \mathrm{argmin}_{\lambda' \geq \lambda} \psi(\lambda').\]
The approach we take to compete with the best choice of $\lambda'$ with respect to the complexity measure $\psi$, is to find an approximate minimizer $\lambda'$ and then run PPA with $\lambda'$. To compute such a $\lambda'$, suppose that we had an access to an oracle that computes $d_{\lambda'}$ for a given parameter $\lambda'>0$. Using the continuity of the effective dimension, we could optimize for $\psi$ over a discrete set of the form $\{\lambda, 2\lambda,\ldots, 2^{C \log\,d} \lambda\}$. The main difficulty stems from the fact that the cost of implementing an oracle that computes $d_{\lambda'}$ already scales with $d_{\lambda'}^{2}d$. We now describe a novel approach for minimizing \cref{eq:compLambda} in total time which is the same upto logarithmic factors as the time taken optimize with the optimal choice of $\lambda's$.

\subsection{Efficient Tuning using Undersampling}
The following theorem which  forms our key contribution shows that an approximate minimizer of the complexity measure $\psi$ can be computed in time proportional to the optimum value of the complexity measure $\psi^*$.
\begin{thm} \label{thm:mainComp}
There exists an algorithm which receives a data matrix $A \in \mathbb{R}^{n \times d}$ and a regularization parameter $\lambda>0$,  and with high probability outputs a regularization parameter $\lambda^{\dagger}$ satisfying
\[
\psi(\lambda^{\dagger}) = O(\psi^\star) \triangleq O\left(\min_{\lambda' \ge \lambda} \left \{\frac{\lambda'}{\lambda}\cdot(\mathrm{nnz}(A) + d_{\lambda'}^2d) \right \} \right)~.
\]
The runtime of the algorithm is $\tilde{O}\left(\psi^*\right)$.
\end{thm}
The above theorem can now be used to immediately prove Theorem \ref{thm:mainCompRuntime}.
\\
\begin{proof}(of \cref{thm:mainCompRuntime})
The proof follows by first obtaining $\lambda^{\dagger}$ using Theorem~\ref{thm:mainComp} and then running PPA with $\lambda^{\dagger}$ with the Sketch-to-Preconditioning  algorithm whose total time is summarized in \eqref{eqn:ppatotaltime}.
\end{proof}
The main idea behind Theorem \ref{thm:mainComp} is that instead of (approximately) computing the effective dimension for each candidate $\lambda'$, we guess the optimal complexity $\psi^\star$ and employ undersampling to test whether a given candidate $\lambda'$ attains the desired complexity. The key ingredient to this approach is described in the following theorem.
\begin{thm}
\label{thm:effectiveEstIter}
Let $A \in \mathbb{R}^{n\times d}$, $\lambda'>0$
and $m \in \mathbb{R}_{>0}$. There exists an algorithm and a universal constant $c$ that verifies whether $d_{\lambda'}(A^\top A)\leq c\cdot m$ in time $\tilde{O}(\mathrm{nnz}(A)+dm^{2})$.
\end{thm}
We provide the proof of Theorem \ref{thm:effectiveEstIter} in Appendix \ref{sec:undersamplingproof}. We now use the above theorem to prove Theorem \ref{thm:mainComp}.
\\
\begin{proof} \textbf{(of Theorem \ref{thm:mainComp})}
Consider the procedure (Algorithm \ref{alg:phistar}) which describes a routine based on a binary search over the possible values of $\psi^*$. Since $\psi(\lambda) \leq \mathrm{nnz}(A) + d^3$ and since $\psi(\lambda') \geq \mathrm{nnz}(A) + d$ for $\lambda' \geq \lambda$, we get that 
\[\psi^* \in [\mathrm{nnz}(A) + d, \mathrm{nnz}(A) + d^3].\] Furthermore, since $\psi(\lambda') \geq \mathrm{nnz}(A) + d^3$ for all $\lambda' \geq d^2\lambda$, we have that \[ \lambda' \in [\lambda, d^2\lambda].\]
Now for any value $\phi$ for which the for loop over $\lambda'$ executes without stopping, it can be seen that $\psi^* \geq \phi/2$. Therefore let $\phi^*$ be the value where it stops and let $\bar{\lambda}$ be the value returned. It can be seen that
\[ \frac{\phi^*}{4} = \frac{\psi(\bar{\lambda})}{4} \leq \psi^* \leq \psi(\bar{\lambda}) = \phi^*,\] which proves the correctness of the algorithm. To bound the running time note that there are at most $O(\log^2(d))$ checks of the form $\psi(\lambda') \leq \phi$. Note that such a check is equivalent to checking whether $d_{\lambda'}^2(A^{\top}A) \leq \frac{1}{d}\left[\frac{\lambda}{\lambda'}\phi - \mathrm{nnz}(A) \right]$ and hence using Theorem \ref{thm:effectiveEstIter} can be implemented in time at most 
\[\tilde{O}(\phi).\]
Furthermore note that for each of the checks done, $\phi \leq 4\psi^*$. Therefore the total time of the procedure is bounded by 
\[\tilde{O}(\psi^*).\]
\begin{algorithm}
\caption{Binary Search for Optimal $\bar{\lambda}$}
\label{alg:phistar} \begin{algorithmic} \State \textbf{Input:}
A matrix $A$, parameter $\lambda$
\For {$\phi \in \{\mathrm{nnz}(A) + d, \mathrm{nnz}(A) + 2d, \mathrm{nnz}(A) + 4d,  \ldots \mathrm{nnz}(A) + d^3\}$}
\For {$\bar{\lambda} \in \{\lambda, 2\lambda, 4\lambda, \ldots d^2\lambda\}$}
\State Check whether $\psi(\bar{\lambda}) \leq \phi$. \Comment{Using Theroem \ref{thm:effectiveEstIter}}
\State If yes, return $\bar{\lambda}$. Else Continue.
\EndFor
\EndFor 
\State Return $\lambda$.
\end{algorithmic} 
\end{algorithm}
\end{proof}


\section*{Acknowledgements} 
We thank Elad Hazan, Ohad Shamir and Christopher Musco for fruitful discussions and valuable suggestions.

\newpage 

\bibliographystyle{plain}
\bibliography{library}

\newpage

\appendix

\section{Proof of Lemma \ref{lemma:expectatationdlambda}} \label{sec:stabilityProof} 

\begin{proof}[Proof of Lemma \ref{lemma:expectatationdlambda}]
Let $\lambda>0$ and denote the spectral decomposition of $C$ by
$\sum_{i=1}^{d}\lambda_{i}u_{i}u_{i}^{\top}$. Let $k=\arg\!\max\{j\in[d]:\,\lambda_{j}\ge\lambda\}.$
Note that 
\[
\forall i\in[k]\,\,\,\,\frac{\lambda_{i}}{\lambda_{i}+\lambda}\ge\frac{\lambda_{i}}{2\lambda_{i}}=\frac{1}{2},\,\,\,\,\forall i>k\,\,\,\frac{\lambda_{i}}{\lambda_{i}+\lambda}\ge\frac{\lambda_{i}}{2\lambda}\text{\,.}
\]

Therefore, 
\begin{equation}
d_{\lambda}(C)=\sum_{i=1}^{k}\frac{\lambda_{i}}{\lambda_{i}+\lambda}+\sum_{i>k}\frac{\lambda_{i}}{\lambda_{i}+\lambda}\ge\frac{1}{2}\left(k+\sum_{i>k}\frac{\lambda_{i}}{\lambda}\right)\label{eqn:dlambdalowerbound}
\end{equation}
Denote the eigenvalues of $\hat{C}=\frac{1}{n}\sum_{i=1}^{n}x_{i}x_{i}^{\top}$
by $\hat{\lambda}_{1},\ldots,\hat{\lambda}_{d}$. Since for any $i\in[d]$,
$\frac{\lambda_{i}}{\lambda_{i}+\lambda}\le1$, we have that 
\begin{equation}
d_{\lambda}(\hat{C})=\sum_{i=1}^{k}\frac{\hat{\lambda}_{i}}{\hat{\lambda}_{i}+\lambda}+\sum_{i>k}\frac{\hat{\lambda}_{i}}{\hat{\lambda}_{i}+\lambda}\le k+\sum_{i>k}\frac{\hat{\lambda}_{i}}{\hat{\lambda}_{i}+\lambda}\leq k+\frac{\sum_{i>k}\hat{\lambda_{i}}}{\lambda}\,.\label{eqn:dlambdaupperbound}
\end{equation}
We now consider the random variable $\sum_{i>k}\hat{\lambda}_{i}$.
To argue about this random variable consider the following identity
which follows from the Courant-Fisher min-max principle for real symmetric
matrices. 
\[
\sum_{i>k}\hat{\lambda}_{i}=\min\left\{ \mathrm{tr}(V^{\top}\hat{C}V):\,\,V\in\mathbb{R}^{d\times k},\,V^{\top}V=I\right\} \,.
\]
Let $U_{i>k}$ be the $d\times d-k$ matrix with the columns $u_{k+1}\ldots u_{d}$.
We now have that 
\begin{align}\label{eqn:expecationboundlambda}
\mathbb{E}[\sum_{i>k}\hat{\lambda}_{i}] & =\mathbb{E}\left[\min\left\{ \mathrm{tr}(V^{\top}\hat{C}V):\,\,V\in\mathbb{R}^{d\times k},\,V^{\top}V=I\right\}\right] \nonumber \\
 & \leq\min\left\{ \mathbb{E}\left[\mathrm{tr}(V^{\top}\hat{C}V)\right]:\,\,V\in\mathbb{R}^{d\times k},\,V^{\top}V=I\right\} \,\nonumber \\
 & \leq\mathbb{E}\left[\mathrm{tr}(U_{i>k}^{\top}\hat{C}U_{i>k})\right]=\mathrm{tr}(U_{i>k}^{\top}CU_{i>k})\nonumber \\
 & =\sum_{i>k}\lambda_{i}
\end{align}
Combining \eqref{eqn:dlambdalowerbound}, \eqref{eqn:dlambdaupperbound}
and \eqref{eqn:expecationboundlambda} and taking expectations we
get that 
\[
\mathbb{E}[d_{\lambda}(\hat{C})]\leq2d_{\lambda}(C)
\]
\end{proof}

\section{Efficient One-side Estimate on The Effective Dimension using Ridge Leverage Score Undersampling}
\label{sec:undersamplingproof}
In this section we prove \cref{thm:effectiveEstIter}.
Inspired by \cite{Cohen2015b, cohen2017input}, our strategy is to use undersampling to obtain sharper estimates to the ridge leverage scores. We start by incorporating an undersampling parameter $\alpha \in (0,1)$ into Definition \ref{def:leverageSampling}.
\begin{defn} \textbf{(Ridge Leverage Score Undersampling)}
Let $(u_i)_{i=1}^n$ be a sequence of ridge leverage score overestimates, i.e., $u_i \ge \tau_{\lambda,i}$ for all $i$. For some fixed positive constant $c>0$ and accuracy parameter $\epsilon$, define $p_i = \min \{1,c \epsilon^{-2} \alpha u_i \log d\}$ for each $i \in [n]$. Let $\texttt{Sample}(u,\epsilon,\alpha)$ denote a function which returns a diagonal matrix $S \in \mathbb{R}_{\ge 0}^{n \times n}$, where $S_{i,i}= ((1+\epsilon) \frac{p_i}{\alpha})^{-1/2}$ with probability $p_i$ and $0$ otherwise.
\end{defn}

Note that while we reduce each probability $p_i$ by factor $\alpha$, the definition of $S_{i,i}$ neglects this modification. Hence, our undersampling is equivalent to sampling according to Definition \ref{def:leverageSampling} and preserving each row with probability $1-\alpha$. By employing undersampling we cannot hope to obtain a constant approximation to the true ridge leverage scores. However, as we describe in the following theorem, this strategy still helps us to sharpen our estimates to the ridge leverage scores.
\begin{thm}
\label{thm:ridgeLeverageUnderSampling}Let $u_i\ge \tau_{\lambda,i}$ for all $i$ and let $\alpha\in(0,1)$ be an undersampling parameter.
Given $S=\texttt{Sample}(u,1/2,\alpha)$, we form new estimates $(u_i^{(\textrm{new})})_{i=1}^n$ by
\begin{equation}
u_i^{(\textrm{new})}:= \min \left \{ a_{i}(  A^\top S^\top S A+\lambda I)^{-1}a_{i},u_i\right \}\,.\label{eq:spectralToRidge}
\end{equation}
Then with high probability, each $u_i^{(\textrm{new})}$ is an overestimate
of $\tau_{\lambda,i}$ and $\|u_i^{(\textrm{new})}\|_{1}\le 3d_{\lambda}/\alpha$.
\end{thm}
The proof of the theorem (which is similar to Theorem 3 of \cite{Cohen2015b} and Lemma 13 of \cite{cohen2017input}) is provided below for completeness. Equipped with this result, we employ the following strategy in order
to verify whether $d_{\lambda}=O(m)$. Applying the lemma with $\alpha=6m/\|u\|_1$,
we have that if $d_{\lambda}\le m$ then $\|\tilde{\tau}_{\lambda}\|_{1}\le n/2\,.$
This gives rise to the following test:
\begin{enumerate}
\item If $\|u_i^{(\textrm{new})}\|_{1}\le m$, accept the hypothesis that
$d_{\lambda}\le m$.
\item If $\|u_i^{(\textrm{new})}\|_{1} \ge \|u\|_1 /2$, reject the hypothesis that $d_{\lambda}\le m$.
\item Otherwise, apply Theorem \ref{thm:ridgeLeverageUnderSampling} to
obtain a new vector of overestimates, $(u_i^{(\textrm{new})})_{i=1}^{n}$.
\end{enumerate}
\begin{proof}
\textbf{(of Theorem \ref{thm:effectiveEstIter}) }Note that the rank
of the matrix $SA$ is $\tilde{O}(m)$ with high probability. Hence, each step of the testing procedure costs $\tilde{O}(\mathrm{nnz}(A)+m^{2}d)$.\footnote{Namely, we can compute $(\tilde{A}\tilde{A}^{\top}+\lambda I)^{-1}$ in time
$O(m^{2}d)$. Thereafter, computing $(\tilde{\tau_{i}})_{i=1}^{n}$ can
be done in time $O(\mathrm{nnz}(A))$.} Since our range of candidate ridge parameters
is of logarithmic size and each test consists of logarithmic number of steps, the theorem follows using the union bound.
\end{proof}

We turn to prove Theorem \cref{thm:ridgeLeverageUnderSampling}. The next lemma intuitively says only a small fraction of $A's$ rows
might have a high leverage score. 
\begin{lem}
\label{lem:sahrpLeverageUni}Let $A\in \mathbb{R}^{n\times d},\,\lambda>0$
and denote by $d_{\lambda}$ the effective dimension of $A.$ For
any $u\in \mathbb{R}_{>0}^{n}$ there exists a diagonal rescsaling matrix
$W\in[0,1]^{n\times n}$ such that for all $i\in[n]$, $\tau_{\lambda, i}(WA)\le u_{i}$
and $\sum_{i:W_{i,i}\neq1}u_{i}\le d_{\lambda}\,.$
\end{lem}
\begin{proof}
We prove the lemma by considering a hypothetical algorithm which constructs
a sequence $(W^{(1)},W^{(2)},\ldots)$ of $n\times n$ diagonal matrices
s.t. $W^{(t)}$ converges to some $W$ which possesses the desired properties.
Initially, the algorithm sets $W^{(1)}=I$. At each time $t>1$ it
modifies a single entry $W_{i,i}$ corresponding to (any) index $i\in[n]$
for which $\tau_{\lambda,i}(WA)>u_{i}$; namely, it chooses $W_{i,i}^{(t+1)}\in(0,1)$
such that $\tau_{\lambda,i}(W^{(t+1)}A)=u_{i}\,.$ It is not hard
to verify the following (e.g., see Lemma 5 and 6 of \cite{Cohen2015b}): 
\begin{itemize}
\item
We can always find $W_{i,i}^{(t+1)}$ such that $\tau_{\lambda,i}(W^{(t+1)}A)=u_i$. 
\item
For any $j\neq i$, $\tau_{\lambda,j}(W^{(t+1)}A)\ge\tau_{\lambda,j}(W^{(t)}A)$. 
\item
Since the entries of $W^{(t)}$ are bounded and monotonically decreasing,
$W^{(t)}\rightarrow W\in[0,1]^{n\times n}$ satisfying $\tau_{\lambda,i}(WX)\le u_{i}$. 
\end{itemize}
It is left to show that indeed, $\sum_{i:W_{i,i}\neq1}u_{i}\le d_{\lambda}\,.$
Let $k$ be the first iteration such that $W_{i,i}^{(k)}\neq0$ for
all $i\in\{j:\,W_{j,j}\neq1\}.$ For each such $i$, consider the
last iteration $k_{i}\le k$ where we reduced $W_{i,i}$ such that
$\tau_{\lambda,i}(W^{k_{i}})=u_{i}$. As was mentioned above, in any
intermediate iteration $t\in\{k_{i}+1,\ldots,k\}$, we could only
increase the $i$-th leverage score. Therefore, $\tau_{\lambda,i}(W^{(k)}A)\ge u_{i}$.
Since $W\preceq I$, it follows that 
\begin{align*}
\sum_{i:W_{i,i}\neq1}u_{i} & \le\sum_{i:W_{i,i}\neq1}\tau_{\lambda,i}(W^{(k)}A)\le\sum_{i=1}^{n}\tau_{\lambda,i}(W^{(k)}A)\,.\\
 & \le tr\left(X^{\top}\left(W^{(k)}\right)^{2}X\left(X^{\top}\left(W^{(k)}\right)^{2}A+\lambda I\right)^{-1}\right)\\
 & \le\sum_{i=1}^{d}\frac{\lambda_{i}(A^{\top}\left(W^{(k)}\right)^{2}A)}{\lambda_{i}(A^{\top}\left(W^{(k)}\right)^{2}A)+\lambda}\\
 & \le\sum_{i=1}^{d}\frac{\lambda_{i}(A^{\top}A)}{\lambda_{i}(A^{\top}A)+\lambda}=d_{\lambda}\,.
\end{align*}
\end{proof}
\begin{proof}
\textbf{(of Theorem \ref{thm:ridgeLeverageUnderSampling}) }By Lemma
\ref{lem:sahrpLeverageUni}, there exists a diagonal matrix $W\in[0,1]^{n\times d}$
satisfying 
\[
\forall i\,\,\,\tau_{\lambda,i}(WA)\le\alpha u_i,\,\,\,\,\,\alpha\sum_{i:W_{i,i}\neq1} u_i\le d_{\lambda}\,.
\]
Therefore, 
\begin{align*}
\sum_{i=1}^{n} u_i^{(\textrm{new})} &\le \sum_{i:W_{i,i}=1} a_{i}^{\top}\left(A^{\top}S^{\top}SA+\lambda I\right)^{-1}a_{i}+\sum_{W_{i,i}\neq1}\tilde{\tau}_{\lambda,i}\\
&\le \sum_{i:~W_{i,i}=1}a_{i}^{\top}\left(A^{\top}S^{\top}SA+\lambda I \right)^{-1}a_{i} +\frac{d_{\text{\ensuremath{\lambda}}}}{\alpha}\,.
\end{align*}
We would like to upper bound the first term in the RHS. Since $W\preceq I$, for all $i\in [n]$,
$$
a_{i}^{\top}(A^{\top}S^{\top}SA+\lambda I)^{-1}a_{i}\le a_{i}^{\top}(A^{\top}WS^{\top}SWA+\lambda I)^{-1}a_{i}
$$
Also, it is clear that for any $i$ for which $W_{i,i}=1$,
\[
a_{i}^{\top}(A^{\top}WS^{\top}SWA+\text{\ensuremath{\lambda}I})^{-1}a_{i}=(WA)_{i,:}^{\top}(A^{\top}WS^{\top}SWA+\lambda)^{-1}(WA)_{i,:}
\]
Finally, since the sampling matrix $S$ was chosen according to $(\alpha u_i)_{i=1}^{n}$,
which form valid overestimates of $(\tau_{\lambda,i}(WA))_{i=1}^n$,
Theorem \ref{thm:ridgeLeverageSampling} implies that with high probability,
\[
\frac{1}{2}A^{\top}W^{2}A\preceq AWS^{\top}SWA\preceq A^{\top}W^{2}A\,.
\]
\[
\Rightarrow(\forall i\in[n])\,\,\,\,(WA)_{i,:}^{\top}(A^{\top}WS^{\top}SWA+\lambda I)^{\dagger}(WA)_{i,:}\le2\tau_{\lambda,i}(WA)
\]
We deduce that 
\[
\sum_{i:W_{i,i} = 1} u_i^{(\textrm{new})} \le 2\sum_{i:W_{i,i} = 1}\tau_{\lambda,i}(WA)\le2\sum_{i=1}^{n}\tau_{\lambda,i}(WA)\le 2d_\lambda\,.
\]
All in all, 
\[
\sum_{i=1}^{n}\tau'_{\lambda,i}\le\frac{3d_\lambda}{\alpha}\,.
\]
\end{proof}

\section{Lower Bound on The Sample Complexity} \label{sec:lowerBound}

We now state a nearly matching lower bound on the sample complexity.
To exhibit a lower bound we consider the special case of ridge regression.
Notably, our lower bound holds for any spectrum specification. The
proof appears in Appendix \ref{sec:lowerBound} 
\begin{thm}
\label{thm:lowerexcessrisk} Given numbers $B>0$ and $\lambda_{1}\ge\ldots\ge\lambda_{d}\ge0$,
define $d_{\lambda}=\sum_{i=1}^{d}\frac{\lambda_{i}}{\lambda_{i}+\lambda}$
and $\Lambda=diag(\lambda_{1}\ldots\lambda_{d})$\footnote{For any $x\in\mathbb{R}^{d}$, $\mathrm{diag}(x)\in\mathbb{R}^{d\times d}$
is a diagonal matrix with $i^{th}$ entry $x_{i}$}. Then for any algorithm there exist a distribution $\mathcal{D}$
over $\mathbb{R}^{d}\times\mathbb{R}$ such that for any algorithm
that returns a linear predictor $\hat{w}$, given $n\geq2d/3$ independent
samples from $\mathcal{D}$, satisfies 
\[
\mathbb{E}_{S\sim\mathcal{D}^{m}}\left[\mathbb{E}_{(x,y)\sim D}\left[\frac{1}{2}(\hat{w}^{T}x-y)^{2}\right]\right]-\min_{w:\|w\|\leq B}\mathbb{E}\left[\frac{1}{2}(\hat{w}^{T}x-y)^{2}\right]\geq\frac{d_{\gamma/(n\cdot B^{2})}}{n}
\]
for any $\gamma$ satisfying 
\begin{equation}
d_{\gamma/(n\cdot B^{2})}-\sum_{i=1}^{d}\left(\frac{\lambda_{i}}{\lambda_{i}+\frac{\gamma}{(n\cdot B^{2})}}\right)^{2}\leq\gamma\label{eqn:lambda_cond}
\end{equation}
\end{thm}

To put the bound achieved by Theorem \ref{thm:lowerexcessrisk} into
perspective we specialize the bound achieved for two popular cases
for eigenvalue profiles defined in \cite{Goel2017}. 
We say that a given eigenvalue profile $\lambda_{1}\ldots\lambda_{d}\geq0$
satisfies 
$(C,p)$ Polynomial Decay if there exists numbers $C,p>0$ such that
$\lambda_{i}\leq Ci^{-p}$. Similarly 
it satisfies $C$-Exponential Decay if there exists a number $C>0$
such that $\lambda_{i}\leq Ce^{-i}$. 
The following table specifies nearly matching upper and lower bounds
for polynomial and exponential decays (see exact statements in Appendix
\ref{sec:decays}). 
\begin{table}[h!]
\centering{}%
\begin{tabular}{|c|c|c|}
\hline 
Decay  & Upper Bound  & Lower Bound \tabularnewline
\hline 
Polynomial Decay (degree $p$) & $O\left(\frac{d_{\frac{\epsilon}{B^{2}}}}{\epsilon}\right)$  & $\Omega\left(\frac{d_{\left(\frac{\epsilon}{B^{2}}\right)^{p/(p+1)}}}{\epsilon}\right)$ \tabularnewline
\hline 
Exponential Decay  & $O\left(\frac{d_{\frac{\epsilon}{B^{2}}}}{\epsilon}\right)$  & $\tilde{\Omega}\left(\frac{d_{\frac{\epsilon}{B^{2}}}}{\epsilon}\right)$ \tabularnewline
\hline 
\end{tabular}
\end{table}

\begin{proof}[Proof of Theorem \ref{thm:lowerexcessrisk}]
Owing to Yao's minimax principle, it is sufficient to exhibit a randomized choice of data distributions against which a deterministic algorithm achieves an excess risk lower bounded as above. To this end consider 
$\mathcal{X}= \{\sqrt{d \lambda_1}\;e_1, \ldots, \sqrt{d \lambda_d}\;e_d\}$ and $\mathcal{Y} = \{-1, 1\}$. Define the randomized choice of data distribution by selecting a vector $\sigma_i \sim \{-1, 1\}^d$ uniformly randomly. The randomized distribution is now defined as first defining the marginal distribution over $x$ as 
\[\mathrm{Pr}(x = e_i) = 1/d \qquad i = 1 \ldots d\]
Further given $\sigma$ the conditional distribution of $y$ is defined as 
\[
Pr\left[y=\pm 1|X=\sqrt{d\lambda_{i}}e_i\right]=\frac{1}{2}(1\pm\sigma_{i}b)\,,
\]
where $b=\sqrt{d/6m}$. Note that $\mathbb{E}[xx^{\top}]=\Lambda$. Consider the following definitions 
\[F(w) \triangleq \mathbb{E}\left[\frac{1}{2}(w^Tx - y)^2\right]\]
Further for any $\lambda > 0$ define
\[F_{\lambda}(w) \triangleq F(w) + \frac{\lambda}{2}\|w\|^2 \qquad w_{\lambda}^* \triangleq \arg\!\min_  {w \in \mathbb{R}^d} F_{\lambda}(w) \]
A straightforward calculation shows that 
\[
(w_{\lambda}^{\star})_{i}=\frac{b}{\sqrt{d}}\frac{\sqrt{\lambda_{i}}}{\lambda_{i}+\lambda}\,\sigma_{i}\,.
\]
Further via complementary slackness we have that there exists some $\lambda^*$ for which $\|w_{\lambda}^*\| = B$. First note that since $\|\hat{w}\| \leq B = \|w_{\lambda}^*\|$, we have that 
\begin{equation}
\label{eqn:reduction}
F(\hat{w})-F(w_{\lambda^*}^{\star})\ge F_{\lambda^*}(\hat{w})-F_{\lambda^*}(w_{\lambda^*}^{\star}).
\end{equation}
Therefore it is sufficient to bound the quantity on the RHS which the following claim shows. 
\begin{lem}
\label{lem:error_lower_bound}
        For any $\lambda > 0$,
        \[\mathbb{E}[F_{\lambda}(\hat{w})]-F_{\lambda}(w_{\lambda}^{\star}) = \Omega\left(\frac{d_{\lambda}}{n}\right)\]
\end{lem}
Further we have that 
\begin{lem}
\label{lem:norm_bound}
        For any $\gamma > 0$ define $\lambda(\gamma) = \frac{\gamma}{n \cdot B^2}$. We have that $\lambda(\gamma) \geq \lambda^*$ if the following holds
        \[d_{\lambda(\gamma)} - \sum_{i=1}^{d} \left( \frac{\lambda_i}{\lambda_i + \lambda(\gamma)} \right)^2 \leq \gamma\]
\end{lem}
Putting together \eqref{eqn:reduction},   \cref{lem:error_lower_bound}, and \cref{lem:norm_bound} gives us that if $\gamma$ satisfies \eqref{eqn:lambda_cond} we have that
\[\mathbb{E}[F(\hat{w})]-F(w_{\lambda^*}^{\star})\ge \mathbb{E}[F_{\lambda^*}(\hat{w})]-F_{\lambda^*}(w_{\lambda^*}^{\star}) \geq \Omega\left(\frac{d_{\lambda^*}}{n}\right) \geq \Omega\left(\frac{d_{\lambda(\gamma)}}{n}\right)\]
which finishes the proof.
\end{proof}
\begin{proof}[Proof of Lemma \ref{lem:norm_bound}]
Note that for any $\gamma$ to ensure that $\lambda(\gamma) \geq \lambda^*$ if and only if $\|w_{\lambda(\gamma)}^*\| \leq B$. To check the latter consider the following
\begin{align*}
\|w_{\lambda(\gamma)}^{*}\|^2 = \sum_{i=1}^{d} \frac{b^2}{d}\cdot \frac{\lambda_i}{(\lambda_i + \lambda(\gamma))^2} = \frac{1}{n \cdot \lambda(\gamma)} \sum_{i=1}^d \left( \frac{\lambda_i}{\lambda_i + \lambda(\gamma)} - \frac{\lambda_i^2}{(\lambda_i + \lambda(\gamma))^2} \right) \leq B^2 
\end{align*}
where the inequality follows from the definition of $\lambda(\gamma)$ and the condition on $\gamma$.
\end{proof}
\begin{proof}[Proof of Lemma \ref{lem:error_lower_bound}]
Consider the following equations 
\begin{align}
\label{eqn:derive1}
\mathbb{E}[F_{\lambda}(\hat{w})-F_{\lambda}(w_{\lambda}^{\star})] & =\mathbb{E}\|\hat{w}-w^{\star}\|_{\Lambda+\lambda}^{2} \nonumber\\
& \ge \mathbb{E}\left[\sum_{i=1}^{\text{d}}(\lambda_{i}+\lambda)\,(w_{\lambda}^*)_i^{2})\,\mathbf{1}_{\hat{w}_{i}(w_{\lambda}^{\star})_i\leq0}\right] \nonumber\\
&=\mathbb{E}\left[\frac{b^2}{d}\sum_{i=1}^{d}\left(\frac{\lambda_{i}}{\lambda_{i}+\lambda}\right)\cdot\mathbf{1}_{\hat{w}_i(w_{\lambda}^{\star})_i\leq0}\right] \nonumber\\
 &=\frac{b^2}{d}\sum_{i=1}^{d}\frac{\lambda_{i}}{\lambda_{i}+\lambda}\cdot \mathrm{Pr}\left[\hat{w}_{i}\cdot (w_{\lambda}^{*})_i\le0\right]\,.
\end{align}
We will now consider the term $\mathrm{Pr}\left[\hat{w}_{i}\cdot (w_{\lambda}^{*})_i\le0 \right]$. Note that since $\sigma_i$ has the same sign as $(w_{\lambda}^{*})_i$, we have that
\begin{align}
\label{eqn:derive2}
\mathrm{Pr}\left[\hat{w}_{i}\cdot (w_{\lambda}^{*})_i\le0\right] &= \mathrm{Pr}[w_i \geq 0 | \sigma_i \leq 0] + \mathrm{Pr}[w_i \leq 0 | \sigma_i \geq 0] \nonumber\\
&= 1 - \mathrm{Pr}[w_i \leq 0 | \sigma_i \leq 0] + \mathrm{Pr}[w_i \leq 0 | \sigma_i \geq 0] \nonumber\\
&\geq 1 - |\mathrm{Pr}[w_i \leq 0 | \sigma_i \leq 0] - \mathrm{Pr}[w_i \leq 0 | \sigma_i \geq 0]| \nonumber\\
&\geq 1 - \frac{1}{2}\sqrt{D_{KL}\left( p\left( S | \sigma_i \leq 0 \right) \; \big|\big| \; p \left( S | \sigma_i \geq 0 \right)\right)}
\end{align}
Since $S$ is composed of $n$ i.i.d instances we can use the chain rule to obtain that
\begin{equation}
        \label{eqn:klchainrule}
        D_{KL}\left( p\left( S | \sigma_i \leq 0 \right) \; \big|\big| \; p \left( S | \sigma_i \geq 0 \right)\right) = m \cdot D_{KL}\left( p\left( (x,y) | \sigma_i \leq 0 \right) \; \big|\big| \; p \left( (x,y) | \sigma_i \geq 0 \right)\right)
\end{equation}
Further we have that
\[p\left( (x,y) | \sigma_i \right) = \frac{1}{d} \cdot p\left( (x,y) | \sigma_i, x = e_i\right) + \left(1 - \frac{1}{d}\right)p\left( (x,y) | \sigma_i, x\neq e_i\right)\]
We can now use the joint convexity of KL divergence to obtain that
\begin{align*}
&D_{KL}\left( p\left( (x,y) | \sigma_i \leq 0 \right) \; \big|\big| \; p \left( (x,y) | \sigma_i \geq 0 \right)\right) \leq  \\
   &     \frac{1}{d} \cdot D_{KL}\left( p\left( (x,y) | \sigma_i \leq 0, x_i = e_i \right) \; \big|\big| \; p \left( (x,y) | \sigma_i \geq 0, x_i = e_i \right)\right)\\ 
      &  + \left( 1 - \frac{1}{d} \right)D_{KL}\left( p\left( (x,y) | \sigma_i \leq 0, x_i \neq e_i \right) \; \big|\big| \; p \left( (x,y) | \sigma_i \geq 0, x_i \neq e_i \right)\right) \qquad \qquad 
\end{align*}
Note that the distribution on $(x,y)$ is independent of $\sigma_i$ conditioned on $x_i \neq e_i$ and therefore the second term above is zero and therefore we have that
\begin{align*}
& D_{KL}(p( (x,y)|\sigma_i \leq 0 )  \,\vert \vert\, p((x,y) | \sigma_i \geq 0 ))  \\
&\leq \frac{1}{d} \cdot D_{KL}( p ( y | \sigma_i \leq 0, x_i = e_i )  \,\vert \vert\, p ( y | \sigma_i \geq 0, x_i = e_i ))
\end{align*}
The RHS now is the KL divergence between two Bernoulli random variables with parameters $\frac{1}{2}(1 + b)$ and $\frac{1}{2}(1 - b)$ respectively. Following arguments similar to \cite{Shamir2015}(Lemma 4) this can be seen to be bounded by $6b^2$ when $b \leq 1/2$. Therefore we have that 
 \begin{equation*}
        D_{KL}\left( p\left( (x,y) | \sigma_i \leq 0 \right) \; \big|\big| \; p \left( (x,y) | \sigma_i \geq 0 \right)\right) \leq \frac{6b^2}{d} 
\end{equation*}
Putting the above together with \eqref{eqn:derive1},\eqref{eqn:derive2} and \eqref{eqn:klchainrule} we get that
\[\mathbb{E}[F_{\lambda}(\hat{w})-F_{\lambda}(w_{\lambda}^{\star})] \geq \frac{b^2}{d}\sum_{i=1}^{d}\frac{\lambda_{i}}{\lambda_{i}+\lambda}\cdot \left( 1 - \sqrt{\frac{3nb^2}{d}}\right) \geq \frac{0.28 \cdot b^2 \cdot d_{\lambda}}{d} = \Omega\left(\frac{d_{\lambda}}{n}\right)\] 
\end{proof}

\begin{proof}[Proof of Corollary \ref{cor:maincorollarylowerbound}]
The proof in both cases follows by choosing $\gamma$ to ensuring that $d_{\gamma /n\cdot B^2} \leq \gamma$ and using Theorem \ref{thm:lowerexcessrisk}. For the case of $(C,p)$-polynomial decay it can be seen using Theorem \ref{thm:dlambdadecaybounds} that the condition is satisfied by choosing $\gamma = \left(\frac{C}{p-1}\right)^{1/(p+1)}\left(\frac{1}{n \cdot B^2}\right)^{\frac{p}{p+1}} + 2$ and in the case of $C$-exponential decay it can be obtained by setting $\gamma = O(\frac{\log(n \cdot B^2) \log(\log(n \cdot B^2))}{n \cdot B^2})$.
\end{proof}

\subsection{Sample Complexity Bounds for Common Decay Patterns} \label{sec:decays}
\begin{thm} \textbf{(\cite{Goel2017})}
\label{thm:dlambdadecaybounds}
        Given an eigenvalue profile $\Lambda = \lambda_1, \ldots \lambda_d \geq 0$, define $d_{\lambda} = \sum_{i = 1}^d \frac{\lambda_i}{\lambda_i + \lambda}$. We have that 
        \begin{itemize}
                \item If $\Lambda$ satisfies $(C,p)$-Polynomial Decay, then $d_{\lambda} \leq \left( \frac{C}{(p-1) \lambda}\right)^{1/p}$.
                \item If $\Lambda$ satisfies $C$-Exponential Decay, then $d_{\lambda} \leq \log\left(\frac{C}{(e-1)\lambda} \right)$. 
         \end{itemize} 
\end{thm}
Combining Theorem \ref{thm:dlambdadecaybounds} with Theorem \ref{thm:lowerexcessrisk} we get the following corollary.
\begin{cor}
\label{cor:maincorollarylowerbound}
        Given numbers $B > 0$ and an eigenvalue profile $\Lambda = \lambda_1 \geq \ldots \geq \lambda_d \geq 0$. Then there exists sets $\mathcal{X}, \mathcal{Y}$ and a distribution over $\mathcal{X}, \mathcal{Y}$ such that for any algorithm that returns a linear predictor $\hat{w}$ given $m \geq 2d/3$ independent samples from $D$ satisfies
         \begin{itemize}
                \item If $\Lambda$ has $(C,p)$-polynomial decay then
                \begin{align*}
&\mathbb{E}_{S \sim \mathcal{D}^m} \left[ \mathbb{E}_{(x,y) \sim D}\left[\frac{1}{2}(\hat{w}^Tx - y)^2\right]\right] - \min_{w: \|w\| \leq B} \mathbb{E}\left[\frac{1}{2}(\hat{w}^Tx - y)^2\right]  \\& \geq \Omega\left(\frac{d_{\left(\frac{1}{n\cdot B^2}\right)^{p/(p+1)} + \frac{1}{n\cdot B^2}}}{n}\right)
\end{align*}
\item If $\Lambda$ has $C$-exponential decay then
\begin{align*}
    &\mathbb{E}_{S \sim \mathcal{D}^n} \left[ \mathbb{E}_{(x,y) \sim D}\left[\frac{1}{2}(\hat{w}^Tx - y)^2\right]\right] - \min_{w: \|w\| \leq B} \mathbb{E}\left[\frac{1}{2}(\hat{w}^Tx - y)^2\right] \geq \\& \ge \Omega\left(\frac{d_{\frac{\log(n \cdot B^2) \log\log(n \cdot B^2)}{n\cdot B^2}}}{n} \right)
\end{align*}
\end{itemize}
\end{cor}

\section{Kernel Ridge-regression} \label{sec:kernel}
Let $\phi:\mathbb{R}^d \rightarrow \mathcal{H}$ be a \emph{feature mapping} into a (possibly infinite-dimensional) Hilbert space $\mathcal{H}$. We consider the minimization of $F(w)=\frac{1}{2}\mathbb{E}_{(x,y)\sim\mathcal{D}}[(\langle w,x \rangle-y)^{2}]$
over a compact and convex subset $\mathcal{W} \subseteq \mathcal{H}$. We may also assume that predictions are bounded by $1$ and denote the diameter of $\mathcal{W}$ by $B$. Since Theorem \ref{thm:excessriskupperbound} does not depend on the intrinsic dimension, we conclude the following.\footnote{While the proof of the theorem refers to the empirical covariance matrix (which requires more careful treatment in the infinite dimensional case), we can always use random features as in \cite{Avron2017a} to enforce finite dimensionality without critically modifying any quantity appearing in our analysis.} 
\begin{cor}
Theorem \ref{thm:excessriskupperbound} holds also in the kernel setting.
\end{cor}
To apply our algorithmic ideas to this task, we consider the dual of the RLM objective:
\[
\hat{F}_\lambda(\alpha) = \frac{1}{2n}\|K\alpha - y\|^2~,
\]
where $K_{i,j}=\langle x_i,x_j \rangle$ is the Gram matrix. Since the eigenvalues of the Gram matrix coincide with those of the empirical covariance matrix, the effective dimension associated with $\frac{1}{n}K$ coincides with the effective dimension of the primal problem. Consequently, applying  preconditioned GD to the dual problem yields the same convergence rate, albeit forming the Gram matrix yields an additional cost of order $n^2$. However, due to our sample complexity bounds, we can bound $n$ by $O(d_\lambda/\epsilon)$. Overall, we obtain the following result.
\begin{cor} \label{cor:kernelRegression}
There exists an algorithm that finds $\epsilon$-approximate minimizer w.r.t. Kernel $\ell_2$-regression in time  
$$
O\left(\min_{\lambda' \ge \lambda} \frac{\lambda'}{\lambda} \left(\frac{d_{\lambda'}^2}{\epsilon^2}+\frac{d_{\lambda'}^3}{\epsilon} \right) \right)
$$
\end{cor}

\section{Additional proofs} \label{sec:technical}

\begin{proof} \textbf{(of Lemma \ref{lem:ridgeScore})}
Fix a scalar $t>0$. By definition, $a_i a_i^\top \preceq t(A^\top A+\lambda I)$ if and only if for every vector $v \in \mathbb{R}^d$, 
$$
v^\top a_i a_i^\top v_i \preceq t v^\top (A^\top A + \lambda I) v
$$
Substituting $v = (A^\top A+\lambda I)^{-1/2}u$, we get that $a_i a_i^\top \preceq t(A^\top A+\lambda I)$ if and only if for every $u \in \mathbb{R}^d$,
$$
u^\top (A^\top A+\lambda I)^{-1/2} a_i a_i^\top (A^\top A+\lambda I)^{-1/2} u \le t 
$$
Using that $(A^\top A+\lambda I) ^{-1/2} a_i a_i^\top (A^\top A+\lambda I) ^{-1/2}$ is a rank-$1$ matrix, the above equivalence can be rewritten as follows:
$$
\begin{aligned}
a_ia_i^\top \preceq t (A^\top A+\lambda I) &\Leftrightarrow (A^\top A+\lambda I) ^{-1/2} a_i a_i^\top (A^\top A+\lambda I) ^{-1/2} \preceq tI \\
& \Leftrightarrow \lambda_1((A^\top A+\lambda I) ^{-1/2} a_i a_i^\top (A^\top A+\lambda I) ^{-1/2}) \le t \\
& \Leftrightarrow \mathrm{tr}((A^\top A+\lambda I) ^{-1/2} a_i a_i^\top (A^\top A+\lambda I) ^{-1/2}) \le t \\
& \Leftrightarrow 
\underbrace{a_i^\top (A^\top A+\lambda I)^{-1/2} (A^\top A+\lambda I)^{-1/2} a_i}_{\tau_{\lambda,i}(A)} \le t~,
\end{aligned}
$$
where the last equivalence follows from the cyclic invariance of the trace. The chain of equivalences implies that $\tau_{\lambda,i}$ is the minimal scalar for which $a_i a_i^\top \preceq t (A^\top A+\lambda I)$.

\end{proof}

\begin{proof} \textbf{(of \cref{lem:sketchToPrecondRecap})}
Since $SA$ is $(1/2,\lambda)$-spectral approximation, the modified objective $\hat{F}_\lambda$ is well-conditioned, i.e. the eigenvalues of its Hessian are bounded from above and below by constants. For such conditioned functions, the convergence rate of Projected Gradient Descent is $O(\log(1/\epsilon))$ (see e.g. \cite{nesterov2004introductory}). Thus the overall runtime is dominated by the cost of computing the spectral approximation. 
\end{proof}

\section{Ridge regression using sketch-and-solve}
We next review a sketch-and-solve approach for ridge regression based on spectral approximation. Consider the problem of ridge regression described in Example \ref{exa:regression}. Let $b=n^{-1/2}(y_1,\ldots,y_n)\in \mathbb{R}^{n\times d}$ and note that the minimizer of the regularized risk is given by 
\[
\hat{w} = \arg\!\min \left \{\frac{1}{2} \|Aw-b\|^2 + \frac{\lambda}{2} \|w\|^2:~w \in \mathcal{W} \right\}~.
\]
Let $\bar{A} \in \mathbb{R}^{n \times (d+1)}$ be the matrix obtained by concatenating the $d$-dimensional vector  $b=\frac{1}{\sqrt{n}}\sum y_i x_i$ with the data matrix $A$. Let $S\bar{A}$ a $(\lambda,\epsilon)$-spectral approximation to $\bar{A}$ and let 
\[
\tilde{w} = \arg\!\min \left \{\frac{1}{2}\|S(Ax-b)\|^2  + \frac{\lambda}{2}\|w\|^2:~w \in \mathcal{W} \right\}
\]
\begin{lem} \label{lem:sketch-and-solve} \textbf{(\cite{woodruff2014sketching}, Theorem 2.14)}
Let $\lambda = \epsilon/B^2$. With high probability, $\tilde{w}$ is an $\epsilon$-approximate minimizer to the regularized risk $\hat{F}_{\lambda}$ defined in Example \ref{exa:regression}. The overall runtime is $\tilde{O}(\mathrm{nnz}(A)+d_{\lambda}^2(\hat{C})\, d \,\mathrm{poly}(1/\epsilon))$.
\end{lem}

\end{document}